\documentclass[11pt]{article}
\usepackage[margin=1in]{geometry}
\usepackage[utf8]{inputenc}  
\usepackage[T1]{fontenc}     
\usepackage{hyperref}       
\usepackage{url}            
\usepackage{booktabs}       
\usepackage{amsfonts}       
\usepackage{nicefrac}       
\usepackage{microtype}      
\usepackage{listings}   
\usepackage{xcolor}     
\definecolor{codegreen}{rgb}{0,0.6,0}
\definecolor{codegray}{rgb}{0.5,0.5,0.5}
\definecolor{codepurple}{rgb}{0.58,0,0.82}
\definecolor{backcolour}{rgb}{0.95,0.95,0.92}

\lstdefinestyle{mystyle}{
    backgroundcolor=\color{backcolour},   
    commentstyle=\color{codegreen},
    keywordstyle=\color{magenta},
    numberstyle=\tiny\color{codegray},
    stringstyle=\color{codepurple},
    basicstyle=\footnotesize\ttfamily,
    breakatwhitespace=false,         
    breaklines=true,                 
    captionpos=b,                    
    keepspaces=true,                 
    numbers=left,                    
    numbersep=5pt,                  
    showspaces=false,                
    showstringspaces=false,
    showtabs=false,                  
    tabsize=2
}

\usepackage{listings}   
\definecolor{codegreen}{rgb}{0,0.6,0}
\definecolor{codegray}{rgb}{0.5,0.5,0.5}
\definecolor{codepurple}{rgb}{0.58,0,0.82}
\definecolor{backcolour}{rgb}{0.95,0.95,0.92}

\usepackage{subcaption}
\usepackage{booktabs} 

\usepackage{graphicx}       
\usepackage{amsmath, amssymb, amsthm}
\usepackage{algorithm}
\usepackage{algpseudocode}
\usepackage{natbib}
\usepackage{enumitem}
\theoremstyle{plain}
\newtheorem{theorem}{Theorem}[section]
\newtheorem{corollary}[theorem]{Corollary}

\theoremstyle{remark}
\newtheorem{remark}[theorem]{Remark}

\theoremstyle{definition}

\newcommand{\Expec}[1]{\mathbb{E}\left[#1\right]}

\title{Diffusion Models under Alternative Noise: Simplified \\Analysis and Sensitivity}
\author{
  Juhyeok Choi\thanks{ Seoul National University, Seoul, 08826, South Korea. \texttt{wrons13@snu.ac.kr}} \and
  Chenglin Fan\thanks{  Seoul National University, Seoul, 08826, South Korea. \texttt{fanchenglin@snu.ac.kr}. \textbf{Corresponding author.} } }

\date{}

\begin{document}

\maketitle

\begin{abstract}
Diffusion models, typically formulated as discretizations of stochastic differential equations (SDEs), have achieved state-of-the-art performance in generative tasks. However, their theoretical analysis often involves complex proofs. In this work, we present a simplified framework for analyzing the Euler--Maruyama discretization of variance-preserving SDEs (VP-SDEs). Using Grönwall’s inequality, we derive a convergence rate of $O(T^{-1/2})$ under standard Lipschitz assumptions, streamlining prior analyses. We then demonstrate that the standard Gaussian noise can be replaced by computationally cheaper discrete random variables (e.g., Rademacher) without sacrificing this convergence guarantee, provided the mean and variance are matched. Our experiments validate this theory, showing that (i) discrete noise achieves sample quality comparable to Gaussian noise provided the variance is matched correctly, and (ii) performance degrades if the noise variance is scaled incorrectly.

\end{abstract}

\section{INTRODUCTION}

Diffusion models are a class of generative models that have recently achieved state-of-the-art performance in tasks such as image synthesis, audio generation, and molecular design. Originally introduced as a probabilistic framework \citep{sohl2015deep}, their modern success was catalyzed by significant practical and theoretical advancements, particularly the formulation as a denoising process \citep{ho2020denoising} and their unification with score-based models through a common stochastic differential equation (SDE) framework \citep{song2021score}. Inspired by non-equilibrium thermodynamics, these models learn to reverse a gradual noising process, transforming pure noise into structured data through a series of denoising steps.

The core idea behind diffusion models is to model the data distribution by simulating a Markovian forward process that incrementally adds Gaussian noise to the data, eventually destroying all structure. A neural network is then trained to approximate the reverse process — learning to recover clean data from noisy observations. Once trained, the model can generate new samples by starting from random noise and iteratively applying the learned denoising steps (\cite{sohl2015deep,ho2020denoising}).

Compared to other generative approaches such as Generative Adversarial Networks (GANs; \citealp{GoodfellowPMXWOCB14}) and Variational Autoencoders (VAEs; \citealp{KingmaW13}), diffusion models are often more stable to train and produce higher-quality samples, albeit with slower generation speed due to their iterative nature. Recent advances have focused on mitigating these limitations. For instance, Denoising Diffusion Implicit Models (DDIM) introduced a more efficient, non-Markovian sampling process \citep{song2021denoising}, while techniques like classifier-free guidance improved sample quality and control without needing an external classifier \citep{ho2022classifier}. More recently, consistency models have enabled high-quality, one-step generation by mapping any point on the SDE trajectory directly to the origin \citep{pmlr-v202-song23a}.


Diffusion models have quickly become a central tool in generative AI research, offering both strong theoretical grounding and impressive empirical results across a wide range of applications. These include text-to-image generation through models like DALL-E 2, Imagen, and the highly efficient Latent Diffusion Models \citep{ramesh2022hierarchical, saharia2022photorealistic}, audio synthesis \citep{kong2021diffwave}, molecular and protein design \citep{trippe2022diffusion, watson2023novo}, video generation \citep{ho2022imagenvideo}, and medical imaging \citep{wolleb2022diffusion}.

The theoretical foundations of diffusion models have been established through their connection to SDEs~\citep{song2021score}. Recent work has initiated rigorous analyses of sampling complexity. For general input distributions, \citet{benton2023} obtained a bound scaling linearly with the dimension \(d\), requiring \(\mathcal{O}(d/\varepsilon^2)\) iterations. \citet{li2024c} refined this in the regime \(T / \log^5 T = \Omega(d^2)\), reducing the complexity to \(\mathcal{O}(d/\varepsilon)\). Under additional smoothness assumptions, the iteration complexity further improves to \(\min\{d,\, d^{2/3}L^{1/3},\, d^{1/3}L\} \cdot \varepsilon^{-2/3}\)~\citep{jiao2025instance}, where \(L\) denotes a relaxed Lipschitz constant of the score function. Other directions include the study of quantized and low-precision noise~\citep{xiao2022discrete}, though without convergence guarantees, as well as polynomial-rate results for more general settings~\citep{chen2023sampling, lee2023convergence, li2023, chen2023improved}. For surveys of recent progress, see~\citet{yang2023diffusion, tang2024tutorial}.

\paragraph{Contributions.} 
In this paper, we make the following contributions:
\begin{itemize}

\item \textbf{Simplified analysis.} We analyze the reverse-time variance-preserving SDE (VP-SDE) using a time-rescaled formulation and Grönwall’s inequality, yielding a cleaner and more accessible proof structure.

    \item \textbf{Noise substitution.} We show that the Gaussian noise used in the reverse sampling process can be replaced with simpler, computationally cheaper discrete noise distributions (e.g., Rademacher) while preserving the same convergence rate. We prove that the error incurred by this substitution is of the same order as the Euler--Maruyama discretization error, provided the alternative noise matches the first two moments of the Gaussian (zero mean and unit variance).


     \item \textbf{Experiments.} We empirically demonstrate that symmetric discrete noise achieves sample quality nearly indistinguishable from Gaussian noise, thereby validating our theoretical findings.
    
\end{itemize}

\section{PRELIMINARIES}
\subsection{Gr\"onwall's Inequality}
Gr\"onwall's inequality(\cite{Gronwall}) is a fundamental tool for establishing bounds in the analysis of differential equations and numerical schemes. It provides a powerful method to bound a function that satisfies a certain type of integral inequality.

\begin{theorem}[Continuous version of Grönwall's Inequality]
    Let $a\ge 0$ and $b,u:[0,T]\rightarrow\mathbb{R}_{+}$ be continuous functions. If
    \begin{align*}
        u(t)\le a + \int_0^t b(s)u(s)ds,
    \end{align*}
    then
    \begin{align*}
        u(T)\le a\exp\left( \int_0^T b(t)dt \right).
    \end{align*}
\end{theorem}

A useful corollary provides a discrete analogue of this result for non-negative sequences, which is essential for analyzing iterative methods.

\begin{corollary}[Discrete version of Grönwall's Inequality]
    Let $a\ge 0$ and $\{u_n\}$, $\{b_n\}$ be non-negative sequences. If
    \begin{align*}
        u_n \le a + \sum_{k=0}^{n-1} b_ku_k,
    \end{align*}
    for all $n\in\mathbb{N}$, then
    \begin{align*}
        u_n \le a \left(1 + \sum_{k=0}^{n-1} b_k \prod _{j=k+1}^{n-1} (1+b_j) \right).
    \end{align*}
\end{corollary}

\subsection{Diffusion Models}
Diffusion models are generative models that transform a simple noise distribution into a complex data distribution by simulating a time-reversed stochastic process. The forward process (also called the diffusion process) gradually adds noise to data, while the reverse process (also called the generative or sampling process) aims to recover data from pure noise.

\paragraph{Forward Process.}
Let $X_0 \sim p_{\text{data}}$ be a sample from the data distribution. The forward process $\{X_t\}_{t \in [0, 1]}$  evolves according to the SDE:
\begin{align*}
    dX_t = -\frac{1}{2} \beta(t) X_t \, dt + \sigma(t) \, dW_t,
\end{align*}
where \( W_t \) is a standard Wiener process, \( \beta(t): [0, 1] \to \mathbb{R}_+ \) is a non-decreasing noise schedule, and we define \( \sigma(t) := \sqrt{\beta(t)} \) following the Variance-Preserving (VP) SDE formulation in(\cite{song2021score}). This process transforms the data distribution into a tractable noise distribution (typically \( \mathcal{N}(0, I) \)) as \( t \to 1 \).

\paragraph{Reverse Process.}
By results from stochastic calculus (e.g., Anderson’s time-reversal of diffusions), the reverse-time process $\{\overline{X}_t\}_{t \in [0, 1]}$ also follows an SDE:
\begin{align*}
    d\overline{X}_t = \left( -\frac{1}{2} \beta(t) \overline{X}_t - \beta(t) \nabla_{x} \log p_t(\overline{X}_t) \right) dt + \sqrt{\beta(t)} d\overline{W}_t,
\end{align*}
where $\overline{W}_t$ is a reverse-time Wiener process and $p_t$ is the marginal distribution of $X_t$ in the forward process.

In practice, the score function $\nabla_{x} \log p_t(x)$ is intractable, so it is approximated by a neural network trained to predict either the score or the noise added during the forward process. In particular, Denoising Diffusion Probabilistic Models (DDPM)(\cite{ho2020denoising}) use a parameterized model $\epsilon_\theta(x, t)$ to estimate the noise in the forward process. Plugging this into the reverse SDE leads to the following approximate generative process:
\begin{align*}
    d\overline{X}_t = \left( 
    \frac{\beta(t)}{\sigma_t} \epsilon_\theta(\overline{X}_t, t) - \frac{\beta(t)}{2} \overline{X}_t \right) dt  + \sqrt{\beta(t)} \, d\overline{W}_t,
\end{align*}
which is the SDE we analyze in the following section. Our goal is to quantify the error when simulating this reverse-time process using the Euler--Maruyama scheme.

\subsection{Prior Works on Non-Gaussian Diffusion}

While the standard diffusion framework relies heavily on Gaussian noise, recent works have explored alternative noise distributions to extend modeling capabilities. One line of research focuses on heavy-tailed distributions to better capture outliers in real-world data. For instance, Denoising L\'{e}vy Probabilistic Models (DLPM) \citep{yoon2023score} and Heavy-Tailed Diffusion Models \citep{pandey2025heavy} replace Gaussian noise with $\alpha$-stable or Student-$t$ distributions, demonstrating improved performance on heavy-tailed datasets where Gaussian priors fail. Similarly, Gamma Diffusion Models \citep{nachmani2021denoising} utilize Gamma noise to better fit data with specific structural constraints.

Another direction generalizes the diffusion process beyond noise entirely. Approaches like Cold Diffusion \citep{bansal2022cold} and Soft Diffusion \citep{daras2022soft} demonstrate that generative behavior can emerge from deterministic degradations, such as blurring or masking, proving that Gaussian randomness is not strictly necessary for the generative mechanism.

\textbf{Our Contribution.} Unlike these works, which fundamentally alter the forward process or data assumptions, our approach retains the standard VP-SDE formulation but targets \textit{computational efficiency in the sampling solver}. While Xiao et al. \citep{xiao2022discrete} explored discrete denoising for low-precision hardware, they lacked rigorous convergence guarantees for the sampling SDE. We bridge this gap by proving that computationally cheaper discrete noise (e.g., Rademacher) can substitute Gaussian noise in the Euler-Maruyama scheme without sacrificing the $\mathcal{O}(T^{-1/2})$ strong convergence rate, provided the first two moments are matched.

\section{Proposed Method}

\subsection{Discretizing the Reverse-Time SDE via Euler–Maruyama}

DDPM can be considered as a discretization of the reverse-time VP-SDE
\begin{align*}
    d\overline{X}_t = \left( 
\frac{\beta(t)}{\sigma_t}\epsilon_\theta\left(\overline{X}_t,t\right) - \frac{\beta(t)}{2}\overline{X}_t \right) dt + \sqrt{\beta(t)} d\overline{W}_t.
\end{align*}

To facilitate the analysis of the discretization error over a fixed interval, we re-parameterize the time variable from $t \in [0,T]$ to $\tau \in [0,1]$ via the transformation $\tau = t/T$, where $T$ is the number of timestep of DDPM. This normalization clarifies the step-size dependency. We obtain
\begin{align*}
    &d\overline{X}_\tau = T\left( 
    \frac{\beta(t)}{\sigma_t}\epsilon_\theta\left(\overline{X}_\tau,t\right) - \frac{\beta(t)}{2}\overline{X}_\tau \right) d\tau + \sqrt{T}\sqrt{\beta(t)} d\overline{W}_\tau \\
&= b\left(\overline{X}_\tau,\tau\right)d\tau + \sigma(\tau)d\overline{W}_\tau.
\end{align*}
The Euler-Maruyama discretization can be written as
\begin{align*}
    \overline{X}_\tau^{(h)} = 
        \overline{X}_{sh}^{(h)}+b\left(\overline{X}_{sh},sh\right)(\tau-sh)+\sigma(sh)(\overline{W}_\tau-\overline{W}_{sh})
\end{align*}
where $\tau\in [(s-1)h,sh]$. Here, $h$ denotes the time interval, which equals $1/T$. With Lipschitz assumptions, we obtain that Euler-Maruyama discretization has the strong convergence rate of order $1/2$.

\begin{theorem}
    Suppose there exists $k>0$ such that
    \begin{align*}
        |b(t,x)-b(s,y)|+|\sigma(t)-\sigma(s)|  &\le k\left(|x-y|+|t-s|^{1/2}\right), \\
        |b(t,x)|+|\sigma(t)| &\le k
    \end{align*}
    for all time $t,s$ and objectives $x,y$. Then there exists $c=c(k,x)>0$ such that
    \begin{align*}
        \left( \mathbb{E}\left[ \sup_{0\le t\le 1} \left| \overline{X}_t^{(h)} -\overline{X}_t \right|^2 \right] \right)^{1/2}\le ch^{1/2}.
    \end{align*}
\end{theorem}
\begin{proof}
    Let $[s] = \lfloor s/h \rfloor h$ denote the start of the discretization interval containing $s$. Then
    \begin{alignat*}{2}
        \overline{X}_\tau^{(h)}-\overline{X}_\tau &=&& \int_\tau^1 \underbrace{ b\left(\overline{X}_{[s]}^{(h)},[s]\right) - b\left(\overline{X}_s^{(h)},s\right) }_{=A}ds  + \int_\tau^1 \underbrace{ b\left(\overline{X}_s^{(h)},s\right) - b\left(\overline{X}_s,s\right) }_{=B} ds  + \int_\tau^1 \underbrace{ \sigma\left([s]\right)-\sigma\left(s\right) }_{=C} d\overline{W}_s.
    \end{alignat*}
    Here, note that
    \begin{align*}
        \Expec{|A|^2}{} &\le k^2 \left(|s-[s]|^{1/2}+\Expec{\left|\overline{X}_s^{(h)}-\overline{X}_{[s]}^{(h)}\right|}{}\right)^2 \le 2k^2 \left(|s-[s]|+\Expec{\left|\overline{X}_s^{(h)}-\overline{X}_{[s]}^{(h)}\right|^2}{}\right), \\
        \Expec{|B|^2}{} &\le k^2 \Expec{\left|\overline{X}_s^{(h)}-\overline{X}_s\right|^2}{}, \\
        \Expec{|C|^2}{} &\le k^2 |s-[s]|.
    \end{align*}
    By Jensen's inequality, we have
    \begin{align*}
        \Expec{\sup_{x\le \tau\le 1} \left|\int_\tau^1 A ds\right|^2}{} \le (1-x)\int_x^1 \Expec{|A|^2}{}ds, \\
        \Expec{\sup_{x\le \tau\le 1} \left|\int_\tau^1 A ds\right|^2}{} \le (1-x)\int_x^1 \Expec{|B|^2}{}ds,
    \end{align*}
    Also, by Doob's maximal inequality, we have
    \begin{align*}
        &\Expec{\sup_{x\le \tau\le 1}\left|\int_\tau^1 \sigma([s])-\sigma(s) d\overline{W}_s\right|^2}{} \le 4\Expec{\int_x^1 \left|\sigma([s])-\sigma(s)\right|^2 ds}{}.
    \end{align*}
    Now, let
    \begin{align*}
        m(x) = \Expec{\sup_{x\le \tau \le 1} \left| \overline{X}_\tau^{(h)} -\overline{X}_\tau \right|^2}{}.
    \end{align*}
    Then
    \begin{align*}
        m(x)
        &\le 4\Big((1-x)\int_x^1 \Expec{|A|^2}{}ds \quad + (1-x)\int_x^1 \Expec{|B|^2}{}ds\quad + 4\int_x^1 \Expec{|C|^2}{}ds\Big) \\
        &\le 40k^2\int_x^1 |s-[s]|+\Expec{\left|\overline{X}_s^{(h)}-\overline{X}_{[s]}^{(h)}\right|}{}ds \quad + 20k^2\int_x^1 \Expec{\left|\overline{X}_s^{(h)}-\overline{X}_s\right|^2}{}ds \\
        &\le 40k^2\int_x^1 ch(1+h)ds+60k^2\int_x^1 m(s)ds.
    \end{align*}
    Hence, for some $c>0$, 
    \begin{align*}
        m(x) \le ch + c\int_{x}^1 m(s)ds.
    \end{align*}
    Applying Gr\"onwall's inequality, we obtain
    \begin{align*}
        m(x) \le che^{c(1-x)} \le che^c.
    \end{align*}
\end{proof}

\subsection{Noise Substitution in Diffusion Models}
In the Euler--Maruyama discretization, the Gaussian noise term can be replaced by a discrete random variable with the same mean and variance. Specifically, in the time-changed Euler--Maruyama scheme, the update rule is
\begin{align*}
    \overline{X}_{\tau - h}^{(h)} = \overline{X}_\tau^{(h)} - b\left( \overline{X}_\tau^{(h)}, \tau \right) h - \sigma(\tau) B_h, \quad B_h \sim \mathcal{N}(0, h).
\end{align*}
We consider replacing the Gaussian noise $B_h$ with a scaled discrete random variable $\sqrt{h} E$, where $E$ has zero mean and unit variance:
\begin{align*}
    \overline{X'}_{\tau - h}^{(h)} = \overline{X'}_\tau^{(h)} - b\left( \overline{X'}_\tau^{(h)}, \tau \right) h - \sigma(\tau) \sqrt{h} E.
\end{align*}

The primary motivation for this substitution lies in the computational overhead of random number generation (RNG) on hardware. Standard Gaussian sampling typically relies on algorithms such as the Box-Muller transform or the Ziggurat method. These approaches often require evaluating transcendental functions (e.g., logarithms, square roots, and trigonometric functions) or involve rejection sampling loops that can cause thread divergence on GPUs. In contrast, sampling from discrete distributions like the Rademacher or Uniform distribution is computationally inexpensive, as it maps directly to the native integer outputs of pseudo-random number generators via simple bitwise masking or linear scaling, bypassing the need for complex function evaluations.

Furthermore, the use of discrete noise—specifically Rademacher noise where $E \in \{-1, +1\}$—offers opportunities for hardware-level optimizations in the arithmetic logic units (ALUs). When the noise component is restricted to binary values, the matrix-vector multiplication associated with noise injection reduces from full floating-point operations (FLOPs) to conditional sign flips (additions and subtractions). For high-dimensional diffusion models deployed on resource-constrained edge devices, FPGAs, or specialized AI accelerators, eliminating these multiplication cycles can yield significant improvements in both latency and energy efficiency without requiring changes to the model weights.

To quantify the error introduced by this replacement, consider the one-step mean squared error:
\begin{align*}
    \mathbb{E}\left[ \left| \epsilon_t \right|^2 \right] 
    &= \sigma^2(\tau) \mathbb{E}\left[ \left| B_h - \sqrt{h} E \right|^2 \right] \\
    &= \sigma^2(\tau) \left( \mathbb{E}[B_h^2] + \mathbb{E}[h E^2] - 2 \mathbb{E}[B_h \sqrt{h} E] \right) \\
    &= \sigma^2(\tau) \left( h + h - 2h \, \mathbb{E}[Z E] \right),
\end{align*}
where $Z \sim \mathcal{N}(0,1)$ and is independent of $E$. Since $\mathbb{E}[Z E] = 0$ under independence, the one-step error simplifies to
\begin{align*}
    \mathbb{E}\left[ \left| \epsilon_t \right|^2 \right] = 2 \sigma^2(\tau) h.
\end{align*}
Applying a discrete version of Gr\"onwall's inequality, we obtain the following bound on the total error:
\begin{align*}
    \left( \mathbb{E}\left[ \sup_{0 \le \tau \le 1} \left| \overline{X}_\tau^{(h)} - \overline{X'}_\tau^{(h)} \right|^2 \right] \right)^{1/2} \le C h^{1/2},
\end{align*}
for some constant $C > 0$ depending on the drift and diffusion coefficients.

Combining this with the Euler--Maruyama error bound using the triangle inequality yields our main result:

\begin{theorem}
    Suppose there exists a constant $k > 0$ such that
    \begin{align*}
        |b(t,x) - b(s,y)| + |\sigma(t) - \sigma(s)| &\le k \left( |x - y| + |t - s|^{1/2} \right), \\
        |b(t,0)| + |\sigma(t)| &\le k,
    \end{align*}
    for all $t,s \in [0,1]$ and $x,y \in \mathbb{R}$. Then there exists a constant $c = c(k,x) > 0$ such that
    \begin{align*}
        \left( \mathbb{E}\left[ \sup_{0 \le t \le 1} \left| \overline{X'}_t^{(h)} - \overline{X}_t \right|^2 \right] \right)^{1/2} \le c h^{1/2}.
    \end{align*}
\end{theorem}

\begin{remark}
We provide some justification for the Lipschitz conditions. A central concept of diffusion models is that with sufficiently many timesteps in the learning phase (the forward direction of the SDE), the process converges to pure noise that no longer depends on the original dataset. For this to hold, we require
\begin{align*}
    \prod_{t=1}^T (1-\beta(t)) 
    = \prod_{s=1}^T \left(1-\frac{1}{T}\sigma(sh)\right) 
    \simeq 0.
\end{align*}
If we assume that the above expression does not depend on $T$, then the Lipschitz condition for $\sigma$ is preserved as $T$ varies. For instance, in the linear schedule for $\beta$,
\begin{align*}
    \beta(t) = \beta_{\min}(T) + \frac{t}{T}\big(\beta_{\max}(T) - \beta_{\min}(T)\big),
\end{align*}
maintaining the same scale of the product requires $\beta_{\min}$ and $\beta_{\max}$ to be proportional to $1/T$. This ensures that the Lipschitz condition of $\sigma$ is preserved. In fact, in this case $\sigma$ does not depend on $T$, and we can naturally establish convergence of the error rate.
\end{remark}

To guarantee that the discretization error is at most $\varepsilon$, we require
\[
c h^{1/2} \leq \varepsilon 
\quad \implies \quad 
h \leq \left( \frac{\varepsilon}{c} \right)^2.
\]
Since $h = 1/T$, the number of discretization steps must satisfy
$
T \geq \left( \frac{c}{\varepsilon} \right)^2.
$
Thus, to achieve an $\mathcal{O}(\varepsilon)$ strong error, it suffices to choose
\[
T = \mathcal{O}\!\left( \frac{1}{\varepsilon^2} \right).
\]


\begin{algorithm}[tb]
   \caption{Euler-Maruyama Sampling with Alternative Noise}
   \label{alg:sampling}
\begin{algorithmic}
   \State {\bfseries Input:} Score model $\epsilon_\theta$, noise schedule $\sigma(\cdot)$, number of steps $T$, noise distribution $p_{noise}$ (e.g., Rademacher)
   \State {\bfseries Output:} Sample $\overline{X}_0$
   \State Initialize $\overline{X}_1 \sim \mathcal{N}(0, I)$
   \State $h \leftarrow 1/T$
   \For{$\tau = 1$ {\bfseries down to} $h$ {\bfseries step} $h$}
       \State Sample $E \sim p_{noise}$ such that $\mathbb{E}[E]=0, \mathbb{V}[E]=1$
       \State Compute drift: $d \leftarrow \left(\frac{\beta(\tau)}{\sigma(\tau)}\epsilon_\theta(\overline{X}_\tau, \tau) - \frac{\beta(\tau)}{2}\overline{X}_\tau \right)$
       \State Update: $\overline{X}_{\tau - h} \leftarrow \overline{X}_\tau - d \cdot h - \sigma(\tau)\sqrt{h} E$
   \EndFor
   \State \textbf{return} $\overline{X}_0$
\end{algorithmic}
\end{algorithm}

\subsection{Extension to the Multidimensional Case}

In the multidimensional setting with dimension \( d \), the analysis proceeds similarly. We consider \( d \) independent noise variables, and let \( E \in \mathbb{R}^d \) be a discrete random vector with mean zero and identity covariance, i.e., \( \mathbb{E}[E] = 0 \), \( \mathbb{E}[E E^\top] = I_d \). The modified update rule becomes
\begin{align*}
    \overline{X'}_{\tau - h}^{(h)} = \overline{X'}_\tau^{(h)} - b\left( \overline{X'}_\tau^{(h)}, \tau \right) h - \sigma(\tau) \sqrt{h} E,
\end{align*}
where \( \sigma(\tau) \in \mathbb{R}^{d \times d} \). The one-step mean squared error grows proportionally to \( h \cdot \mathrm{Tr}(\sigma(\tau) \sigma(\tau)^\top) \), and hence, by applying standard techniques including the discrete Gr\"onwall inequality, we obtain the strong approximation error
\begin{align*}
    \left( \mathbb{E} \left[ \sup_{0 \le \tau \le 1} \left\| \overline{X}_\tau^{(h)} - \overline{X'}_\tau^{(h)} \right\|^2 \right] \right)^{1/2} \le c d^{1/2} h^{1/2},
\end{align*}
for some constant $c > 0$ depending on the Lipschitz constants of \( b \) and \( \sigma \), the initial condition, and the bound on the noise. In this setting, we therefore obtain the strong convergence rate \( d^{1/2} h^{1/2} \).
Thus, to achieve an \( \mathcal{O}(\varepsilon) \) strong error, the number of steps $T$ must be chosen such that:
\[
T = \mathcal{O}\left( \frac{d}{\varepsilon^2} \right).
\]
For clarification, we restate our theorem below.

\begin{theorem}\label{thm:mul3}
    Consider the multi-dimensional VP-SDE
    \begin{align*}
        d\overline{X}_\tau = b(\overline{X}_\tau,\tau)d\tau + \sigma(\tau)d\overline{W}_\tau
    \end{align*}
    where $\overline{X}_\tau\in \mathbb{R}^d$, $b:\mathbb{R}^d\times[0,1]\to \mathbb{R}^d$ and $\sigma:[0,1]\to \mathbb{R}^{d\times d}$. Let
    \begin{align*}
        \overline{X}_\tau^{(h)} = 
        \overline{X}_{sh}^{(h)}+b\left(\overline{X}_{sh},sh\right)(\tau-sh)+\sigma(sh)(\overline{W}_\tau-\overline{W}_{sh}),\quad \tau\in[(s-1)h,sh].
    \end{align*}
    be its Euler-Maruyama discretization. Suppose there exists $k>0$ such that
    \begin{align*}
        \|b(x,t)-b(y,s)\|+\|\sigma(t)-\sigma(s)\|_F  &\le k\left(\|x-y\|+|t-s|^{1/2}\right), \\
        \|b(t,x)\|+\|\sigma(t,x)\|_F &\le k
    \end{align*}
    for all time $t,s$ and objectives $x,y$, where $\|\cdot\|_F$ is the Frobenius norm. Then there exists $c=c(k,x)>0$ such that
    \begin{align*}
        \left( \mathbb{E}\left[ \sup_{0\le t\le 1} \left\| \overline{X}_t^{(h)} -\overline{X}_t \right\|^2 \right] \right)^{1/2}\le cd^{1/2}h^{1/2}.
    \end{align*}
\end{theorem}
\begin{proof}
    Let $[s] = \lfloor s/h \rfloor h$ denote the start of the discretization interval containing $s$. Then
    \begin{alignat*}{2}
        \overline{X}_\tau^{(h)}-\overline{X}_\tau &=&& \int_\tau^1 \underbrace{ b\left(\overline{X}_{[s]}^{(h)},[s]\right) - b\left(\overline{X}_s^{(h)},s\right) }_{=A}ds + \int_\tau^1 \underbrace{ b\left(\overline{X}_s^{(h)},s\right) - b\left(\overline{X}_s,s\right) }_{=B} ds  + \int_\tau^1 \underbrace{ \sigma\left([s]\right)-\sigma\left(s\right) }_{=C} d\overline{W}_s.
    \end{alignat*}
    Here, note that
    \begin{align*}
        \Expec{\|A\|^2}{} &\le k^2 \left(|s-[s]|^{1/2}+\Expec{\left\|\overline{X}_s^{(h)}-\overline{X}_{[s]}^{(h)}\right\|}{}\right)^2 \le 2k^2 \left(|s-[s]|+\Expec{\left\|\overline{X}_s^{(h)}-\overline{X}_{[s]}^{(h)}\right\|^2}{}\right) \\
        \Expec{\|B\|^2}{} &\le k^2 \Expec{\left\|\overline{X}_s^{(h)}-\overline{X}_s\right\|^2}{}.
    \end{align*}
    We observe that the mean square displacement of the single step is small enough:
    \begin{align*}
        \Expec{\left\|\overline{X}_s^{(h)}-\overline{X}_{[s]}^{(h)}\right\|^2} \le 2\Expec{\|b([s])\|^2 h^2}{}+2\Expec{\|\sigma([s])\|_F^2dh}\le 2k^2h^2+2k^2dh.
    \end{align*}
    By It\'{o} isometry, we have
    \begin{align*}
        \Expec{\|C\|^2}{} &= \int_\tau^1 \Expec{\|\sigma([s])-\sigma(s)\|_F^2}ds \le k^2 |s-[s]|\le k^2 h.
    \end{align*}
    By Jensen's inequality, we have
    \begin{align*}
        \Expec{\sup_{x\le \tau\le 1} \left\|\int_\tau^1 A ds\right\|^2}{} \le (1-x)\int_x^1 \Expec{\|A\|^2}{}ds, \\
        \Expec{\sup_{x\le \tau\le 1} \left\|\int_\tau^1 B ds\right\|^2}{} \le (1-x)\int_x^1 \Expec{\|B\|^2}{}ds,
    \end{align*}
    Also, by Doob's maximal inequality, we have
    \begin{align*}
        \Expec{\sup_{x\le \tau\le 1}\left\|\int_\tau^1 \sigma([s])-\sigma(s) d\overline{W}_s\right\|^2}{} \le 4\Expec{\int_x^1 \left\|\sigma([s])-\sigma(s)\right\|_F^2 ds}{}\le 4k^2 h.
    \end{align*}
    Now, let
    \begin{align*}
        m(x) = \Expec{\sup_{x\le \tau \le 1} \left\| \overline{X}_\tau^{(h)} -\overline{X}_\tau \right\|^2}{}.
    \end{align*}
    Then
    \begin{align*}
        m(x)
        &\le 3\left( \Expec{\sup_{x\le \tau\le 1} \left\| \int_\tau^1 A ds \right\|^2} +\Expec{\sup_{x\le \tau\le 1} \left\| \int_\tau^1 B ds \right\|^2} + 4\Expec{\int_x^1 \|\sigma([s])-\sigma(s)\|_F^2 ds} \right) \\
        &\le 3\left((1-x)\int_x^1 \Expec{\|A\|^2}{}ds + (1-x)\int_x^1 \Expec{\|B\|^2}{}ds + 4\int_x^1 \Expec{\|C\|^2}{}ds\right) \\
        &\le 6k^2\int_x^1 |s-[s]|+\Expec{\left\|\overline{X}_s^{(h)}-\overline{X}_{[s]}^{(h)}\right\|^2}{}ds + 3k^2\int_x^1 \Expec{\left\|\overline{X}_s^{(h)}-\overline{X}_s\right\|^2}{}ds + 12k^2h \\
        &\le 6k^2(h+2k^2h^2+2k^2dh)+3k^2\int_x^1 m(s)ds + 12k^2h \\
        &\le (12k^4d+18k^2+12k^4h)h+3k^2\int_x^1 m(s)ds.
    \end{align*}
    Hence, for some constant $c=c(k)>0$, 
    \begin{align*}
        m(x) \le cdh + c\int_{x}^1 m(s)ds.
    \end{align*}
    Applying Gr\"onwall's inequality, we obtain
    \begin{align*}
        m(x) \le cdhe^{c(1-x)} \le cdhe^c.
    \end{align*}
\end{proof}

Theorem~\ref{thm:mul3} establishes the $O(h^{1/2})$ convergence rate for the Euler-Maruyama discretization under standard Gaussian increments, practical implementations often utilize discrete random variables to simulate noise. This raises the question of whether the convergence rate is preserved when the Gaussian assumption is relaxed. The following theorem addresses this by considering a process $\overline{X'}_t^{(h)}$ driven by general discrete noise, utilizing the result of Theorem~\ref{thm:mul3} as a foundational bound.

\begin{theorem}
    Suppose there exists a constant $k > 0$ such that
    \begin{align*}
        \|b(t,x) - b(s,y)\| + \|\sigma(t) - \sigma(s)\|_F &\le k \left( \|x - y\| + |t - s|^{1/2} \right), \\
        \|b(t,x)\| + \|\sigma(t)\|_F &\le k,
    \end{align*}
    for all $t,s \in [0,1]$ and $x,y \in \mathbb{R}^d$. Then there exists a constant $C = C(k,x) > 0$ such that
    \begin{align*}
        \left( \mathbb{E}\left[ \sup_{0 \le t \le 1} \left\| \overline{X'}_t^{(h)} - \overline{X}_t \right\|^2 \right] \right)^{1/2} \le Cd^{1/2}h^{1/2}.
    \end{align*}
\end{theorem}
\begin{proof}
    We aim to bound the error between the process driven by discrete noise, $\overline{X'}^{(h)}$, and the true solution $\overline{X}$. By the triangle inequality:
    \begin{align*}
        \left(\mathbb{E}\left[ \sup_{0 \le t \le 1} \left| \overline{X'}_t^{(h)} - \overline{X}t \right|^2 \right]\right)^{1/2} &\le \underbrace{\left(\mathbb{E}\left[ \sup_{0 \le t \le 1} \left| \overline{X'}t^{(h)} - \overline{X}t^{(h)} \right|^2 \right]\right)^{1/2}}_{(I:\text{Noise Coupling (via KMT)})} \\
        &+ \underbrace{\left(\mathbb{E}\left[ \sup_{0 \le t \le 1} \left| \overline{X}_t^{(h)} - \overline{X}t \right|^2 \right]\right)^{1/2}}_{(II: \text{Result from Theorem~\ref{thm:mul3})}}.\end{align*}From  Theorem~\ref{thm:mul3}), we have established that $(II) \le c_1 d^{1/2}h^{1/2}$. We now focus on bounding $(I)$, the distance between the Euler-Maruyama approximation driven by Gaussian noise ($W$) and the one driven by discrete noise ($E$).The discretized processes satisfy:
\begin{align*}
    \overline{X}_{k+1}^{(h)} &= \overline{X}_{k}^{(h)} + b(\overline{X}_{k}^{(h)})h + \sigma(kh)\Delta W_k, \quad \text{where } \Delta W_k \sim \mathcal{N}(0, hI_d) \\
    \overline{X'}_{k+1}^{(h)} &= \overline{X'}_{k}^{(h)} + b(\overline{X'}_{k}^{(h)})h + \sigma(kh)\sqrt{h}E_k, \quad \text{where } \mathbb{E}[E_k]=0, \mathrm{Cov}[E_k]=I_d.
\end{align*}

A crucial result in probability theory (the Skorokhod embedding theorem or the KMT approximation) states that for a random walk $S_n = \sum_{i=1}^n E_i$ with zero mean and finite moments, one can construct a Brownian motion $W$ on the same probability space such that the partial sums approximate the Brownian path. Specifically, there exists a coupling such that:
\begin{align*}
    \mathbb{E}\left[ \sup_{0 \le k \le 1/h} \left\| \sum_{i=0}^{k-1} \sqrt{h}E_i - W_{kh} \right\|^2 \right] \le C_E d h
\end{align*}
for some constant $C_E$. Let $M_k = \sum_{i=0}^{k-1} \sqrt{h}E_i$ and $W_k = W_{kh}$. The difference between the noise paths is small.

Let $\delta_k = \overline{X'}_{k}^{(h)} - \overline{X}_{k}^{(h)}$. Subtracting the update equations:
\begin{align*}
    \delta_{k+1} &= \delta_k + \left( b(\overline{X'}_{k}^{(h)}) - b(\overline{X}_{k}^{(h)}) \right)h + \sigma(kh)(\sqrt{h}E_k - \Delta W_k).
\end{align*}
We can rewrite the noise term accumulation. Let $Z_k = \sum_{i=0}^{k-1} (\sqrt{h}E_i - \Delta W_i)$ be the difference in the driving noise paths. Summing the recurrence:
\begin{align*}
    \delta_n = \sum_{k=0}^{n-1} \left( b(\overline{X'}_{k}^{(h)}) - b(\overline{X}_{k}^{(h)}) \right)h + \sum_{k=0}^{n-1} \sigma(kh)(\sqrt{h}E_k - \Delta W_k).
\end{align*}
Using summation by parts (Abel transformation) on the noise term:
\begin{align*}
    \sum_{k=0}^{n-1} \sigma(kh)(Z_{k+1} - Z_k) &= \sigma((n-1)h)Z_n - \sum_{k=0}^{n-1} (\sigma(kh) - \sigma((k-1)h)) Z_k.
\end{align*}
Using the Lipschitz bound on $\sigma$ and $b$, and the noise approximation bound $\mathbb{E}[\sup \|Z_k\|^2] \le Cdh$, we proceed. Taking norms and expectations:
\begin{align*}
    \|\delta_n\| &\le h \sum_{k=0}^{n-1} k_L \|\delta_k\| + \|\text{NoiseTerm}_n\|.
\end{align*}
Squaring and using Cauchy-Schwarz:
\begin{align*}
    \|\delta_n\|^2 &\le C \left( \sum_{k=0}^{n-1} \|\delta_k\|^2 h + \sup_{k} \|Z_k\|^2 \right).
\end{align*}
Applying the discrete Grönwall inequality:
\begin{align*}
    \sup_{n} \mathbb{E}[\|\delta_n\|^2] &\le C \mathbb{E}\left[\sup_k \|Z_k\|^2\right] \exp(C).
\end{align*}
Since $\mathbb{E}[\sup \|Z_k\|^2] \le \mathcal{O}(dh)$ via the embedding theorem:
\begin{align*}
    \left(\mathbb{E}\left[ \sup_{0 \le t \le 1} \left\| \overline{X'}_t^{(h)} - \overline{X}_t^{(h)} \right\|^2 \right]\right)^{1/2} \le C d^{1/2} h^{1/2}.
\end{align*}
Combining this with term $(II)$, the total error is bounded by $(C + c_1) d^{1/2} h^{1/2}$.
\end{proof}

\section{EXPERIMENTS}
All experiments were conducted on a personal computer equipped with an NVIDIA RTX 4070 GPU.
We evaluate noise substitution and zigzag sampling in diffusion models through a series of experiments. Figures~\ref{fig:various_noises} and~\ref{fig:noise_scaling} illustrate representative sample results, highlighting how noise type and noise scaling influence generation quality.

\begin{figure}[htbp]
    \begin{subfigure}{.5\textwidth}
        \centering
        \includegraphics[width=.9\linewidth]{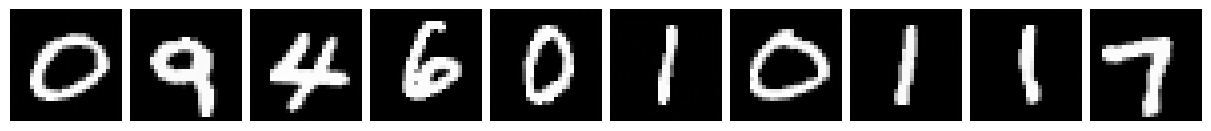}
        \caption{Gaussian Noise}
        \label{fig:sfig1}
    \end{subfigure}
    \begin{subfigure}{.5\textwidth}
        \centering
        \includegraphics[width=.9\linewidth]{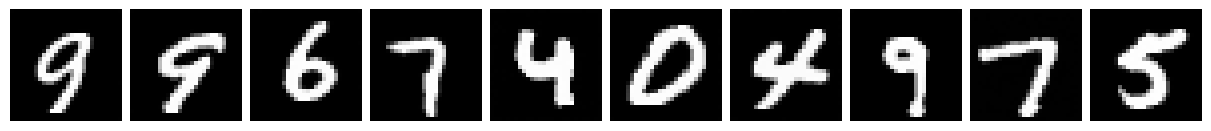}
        \caption{Discrete Gaussian}
        \label{fig:sfig1}
    \end{subfigure}
    \begin{subfigure}{.5\textwidth}
        \centering
        \includegraphics[width=.9\linewidth]{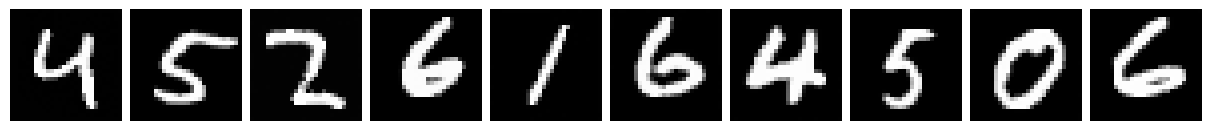}
        \caption{Uniform}
        \label{fig:sfig1}
    \end{subfigure}
    \begin{subfigure}{.5\textwidth}
        \centering
        \includegraphics[width=.9\linewidth]{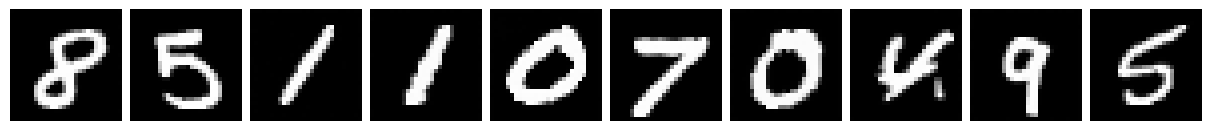}
        \caption{Rademacher Noise}
        \label{fig:sfig1}
    \end{subfigure}
    \begin{subfigure}{.5\textwidth}
        \centering
        \includegraphics[width=.9\linewidth]{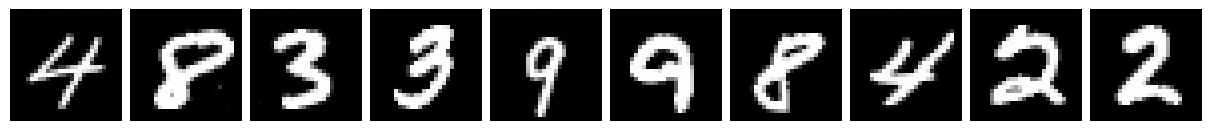}
        \caption{Laplace Noise}
        \label{fig:sfig1}
    \end{subfigure}
    \begin{subfigure}{.5\textwidth}
        \centering
        \includegraphics[width=.9\linewidth]{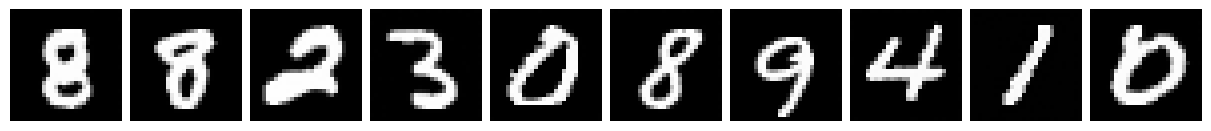}
        \caption{Triangular Noise}
        \label{fig:sfig1}
    \end{subfigure}
    \begin{subfigure}{.5\textwidth}
        \centering
        \includegraphics[width=.9\linewidth]{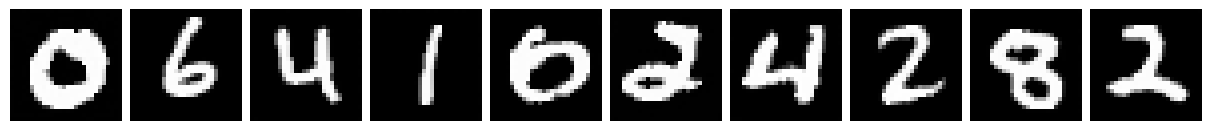}
        \caption{Arcsine Noise}
        \label{fig:sfig1}
    \end{subfigure}
    \caption{MNIST Image Samples with Various Noises}
    \label{fig:various_noises}
\end{figure}

\subsection{Performance of Various Noises}

\begin{table}[h]
    \caption{FID Scores with Various Noises} \label{sample-table}
    \begin{center}
            \begin{tabular}{lll}
            \textbf{NOISE}  &\textbf{FID(MNIST)} &\textbf{FID(CIFAR-10)} \\
            \hline \\
            Gaussian & 2.994326 & 14.642573 \\
            Discrete Gaussian & 3.005654 & 14.672305 \\
            Uniform & 2.972102 & 14.592020 \\
            Rademacher & 2.987234 & 14.540516 \\
            Laplace & 5.766727 & 34.587563 \\
            Triangular & 2.960605 & 14.503871 \\
            Arcsine & 2.987144 & 14.631359 \\
        \end{tabular}
    \end{center}
\end{table}
\textbf{Setup:} We evaluate how different noise distributions affect the sampling quality of a Denoising Diffusion Probabilistic Model (DDPM) trained on the MNIST and CIFAR10 dataset. We test several symmetric noise distributions—Gaussian, Uniform, Rademacher, Triangular, Arcsine—and one asymmetric distribution, Laplace. All noises are scaled to have zero mean and unit variance to match the standard Gaussian noise used in training. Additionally, we include Discrete Gaussian noise, a discrete random variable whose probability mass function is proportional to the probability density of a standard normal distribution evaluated at integer points. Sample quality is assessed using the Fr\'echet Inception Distance (FID) score.

\noindent \textbf{Result:} The empirical results, summarized in Table~\ref{sample-table}, demonstrate that most symmetric noise distributions produce sample quality comparable to standard Gaussian noise across both datasets. On MNIST, the FID scores for Uniform, Rademacher, Triangular, and Arcsine noise are all within a negligible margin of the Gaussian baseline. This trend holds for the more complex CIFAR-10 dataset, where Rademacher and Arcsine noise achieve parity with Gaussian noise (e.g., $14.54$ vs. $14.64$). We note that the absolute FID scores for CIFAR-10 are higher than current state-of-the-art benchmarks. This is because we employed a lightweight model architecture and a fixed, limited training budget to facilitate a broad ablation study on academic hardware. However, the key finding is the relative performance consistency: the substitution of Gaussian noise with computationally cheaper symmetric alternatives does not degrade generation quality. In contrast, the Laplace distribution—the only asymmetric noise tested—resulted in significantly higher FID scores on both datasets ($5.77$ on MNIST and $34.59$ on CIFAR-10), confirming that while the specific shape of the distribution is flexible, symmetry is a critical property for minimizing discretization error.

\begin{figure}[htbp]
    \begin{subfigure}{.5\textwidth}
        \centering
        \includegraphics[width=.9\linewidth]{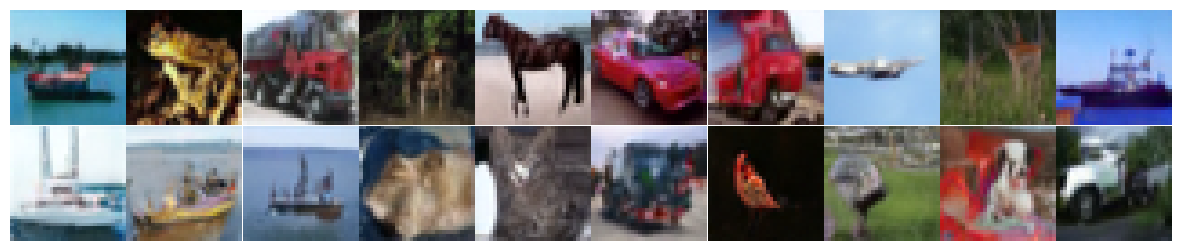}
        \caption{Gaussian Noise}
        \label{fig:sfig1}
    \end{subfigure}
    \begin{subfigure}{.5\textwidth}
        \centering
        \includegraphics[width=.9\linewidth]{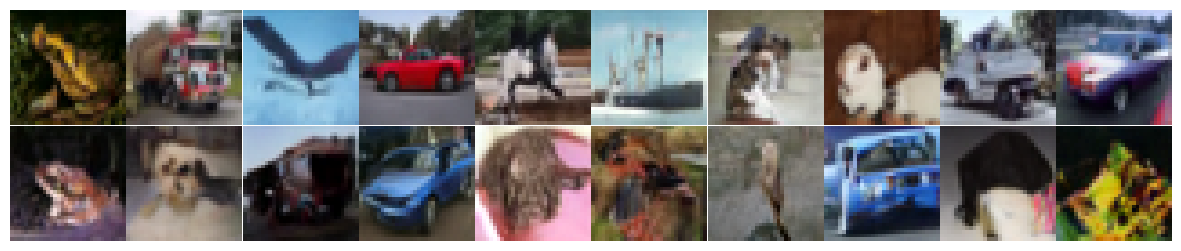}
        \caption{Discrete Gaussian}
        \label{fig:sfig1}
    \end{subfigure}
    \begin{subfigure}{.5\textwidth}
        \centering
        \includegraphics[width=.9\linewidth]{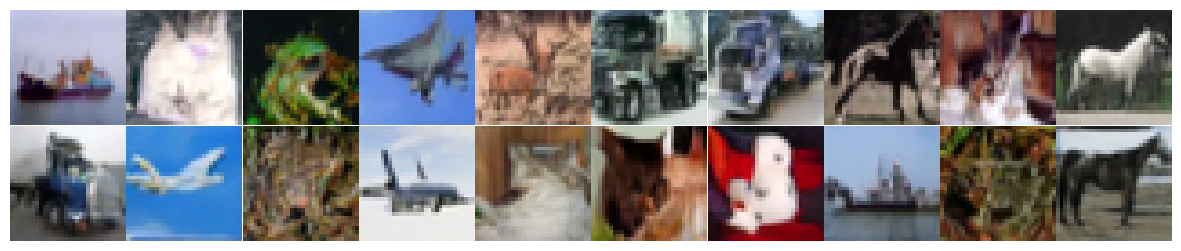}
        \caption{Uniform}
        \label{fig:sfig1}
    \end{subfigure}
    \begin{subfigure}{.5\textwidth}
        \centering
        \includegraphics[width=.9\linewidth]{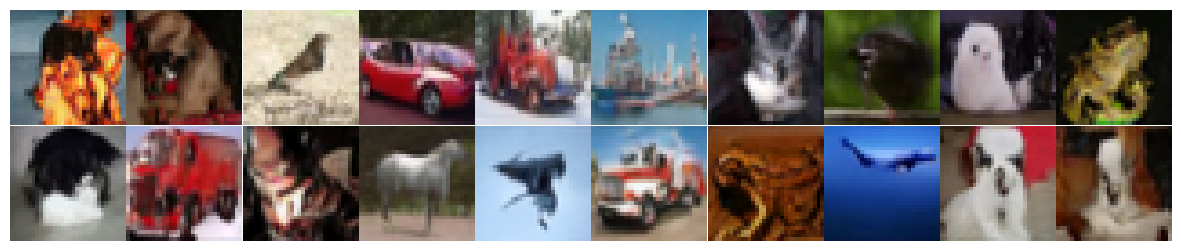}
        \caption{Rademacher Noise}
        \label{fig:sfig1}
    \end{subfigure}
    \begin{subfigure}{.5\textwidth}
        \centering
        \includegraphics[width=.9\linewidth]{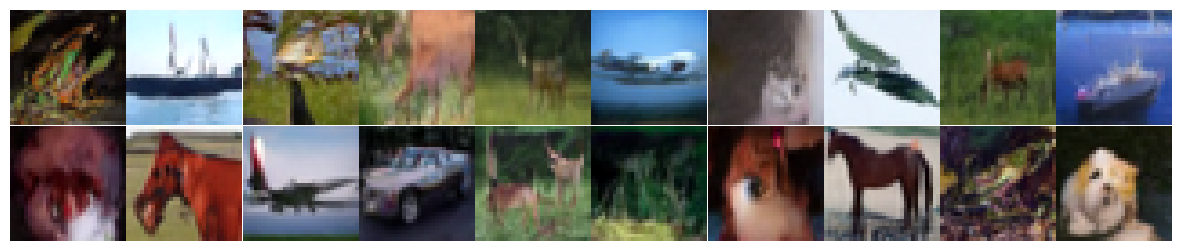}
        \caption{Laplace Noise}
        \label{fig:sfig1}
    \end{subfigure}
    \begin{subfigure}{.5\textwidth}
        \centering
        \includegraphics[width=.9\linewidth]{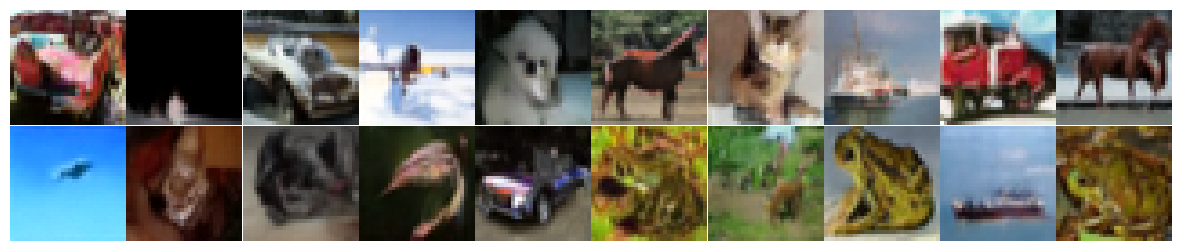}
        \caption{Triangular Noise}
        \label{fig:sfig1}
    \end{subfigure}
    \begin{subfigure}{.5\textwidth}
        \centering
        \includegraphics[width=.9\linewidth]{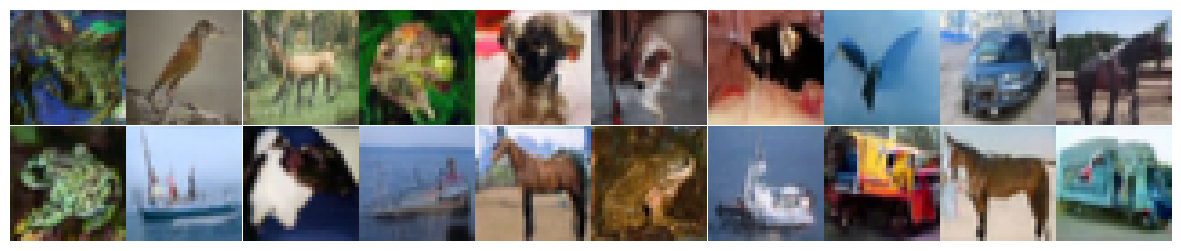}
        \caption{Arcsine Noise}
        \label{fig:sfig1}
    \end{subfigure}
    \caption{CIFAR-10 Image Samples with Various Noises}
    \label{fig:various_noises_CIFAR10}
\end{figure}

\subsection{Noise Scale Ablation}

\textbf{Setup:} To study the effect of noise variance on sample quality, we conducted an ablation study replacing standard Gaussian noise with a Rademacher distribution, where components are drawn from $\{-\alpha,\alpha\}$. This setup provides direct control over the variance through the scalar $\alpha$ while preserving the zero-mean and symmetric properties of the noise. We evaluated generated samples using the Fr\'echet Inception Distance (FID) across a range of $\alpha$ values, hypothesizing an optimum at $\alpha=1.0$ to match the unit variance of the training noise.

\textbf{Result:} The results in Table 2 confirm that noise variance is critical to the synthesis process. The FID score exhibits a sharp, convex relationship with the noise scale $\alpha$, reaching a minimum at the hypothesized optimum of $\alpha=1.0$. Insufficient variance ($\alpha<1.0$) leads to mode collapse and blurry images, whereas excessive variance ($\alpha>1.0$) results in residual noise and artifacts. These findings confirm that matching the statistical moments of the sampling noise to the training distribution is essential for high-fidelity generation.

\begin{table}[h]
    \caption{Sampling with Scaled Rademacher Noise} \label{sample-table2}
    \begin{center}
            \begin{tabular}{ll}
            \textbf{NOISE SCALE}  &\textbf{FID SCORE} \\
            \hline \\
            0.1 & 97.936673 \\
            0.2 & 77.937531 \\
            0.4 & 47.992128 \\
            0.8 & 11.771176 \\
            0.9 & 6.159960 \\
            1.0 & 2.987234 \\
            1.1 & 4.307226 \\
            1.2 & 12.735438 \\
            1.3 & 30.257889 \\
            1.5 & 90.242229 \\
        \end{tabular}
    \end{center}
\end{table}

\begin{figure}[htbp]
    \begin{subfigure}{.5\textwidth}
        \centering
        \includegraphics[width=.9\linewidth]{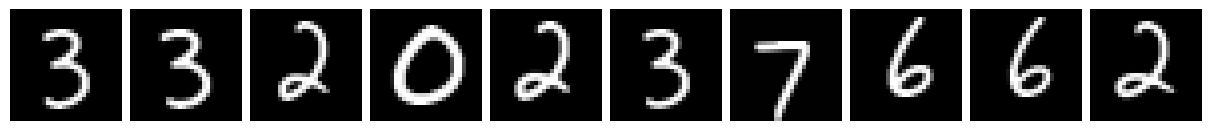}
        \caption{Scale = 0.1}
        \label{fig:sfig1}
    \end{subfigure}
    \begin{subfigure}{.5\textwidth}
        \centering
        \includegraphics[width=.9\linewidth]{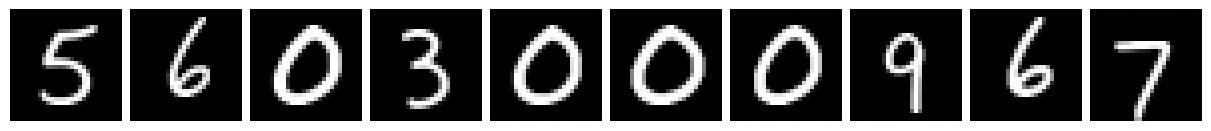}
        \caption{Scale = 0.2}
        \label{fig:sfig1}
    \end{subfigure}
    \begin{subfigure}{.5\textwidth}
        \centering
        \includegraphics[width=.9\linewidth]{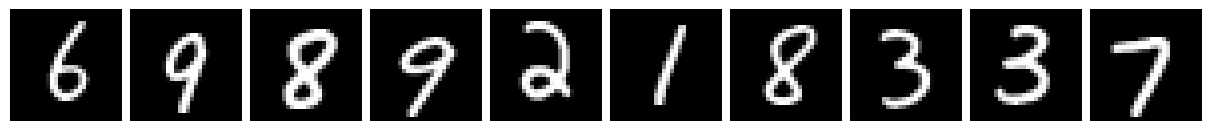}
        \caption{Scale = 0.5}
        \label{fig:sfig1}
    \end{subfigure}
    \begin{subfigure}{.5\textwidth}
        \centering
        \includegraphics[width=.9\linewidth]{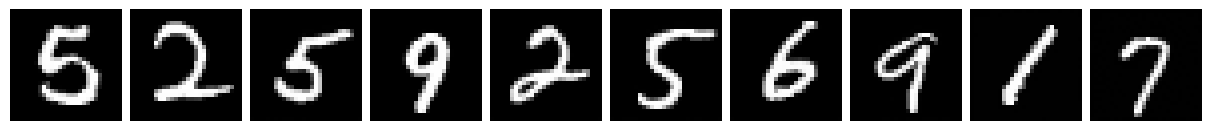}
        \caption{Scale = 1.0}
        \label{fig:sfig1}
    \end{subfigure}
    \begin{subfigure}{.5\textwidth}
        \centering
        \includegraphics[width=.9\linewidth]{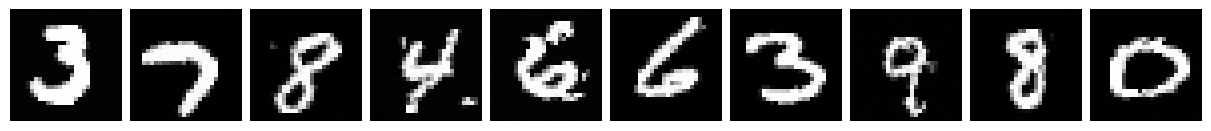}
        \caption{Scale = 1.2}
        \label{fig:sfig1}
    \end{subfigure}
    \begin{subfigure}{.5\textwidth}
        \centering
        \includegraphics[width=.9\linewidth]{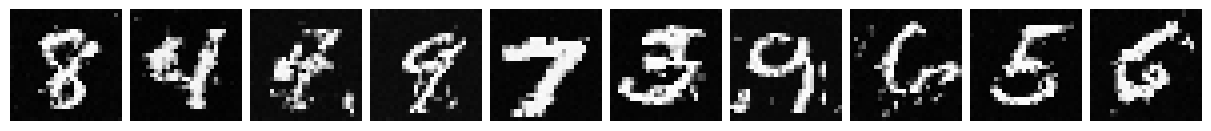}
        \caption{Scale = 1.5}
        \label{fig:sfig1}
    \end{subfigure}
    \begin{subfigure}{.5\textwidth}
        \centering
        \includegraphics[width=.9\linewidth]{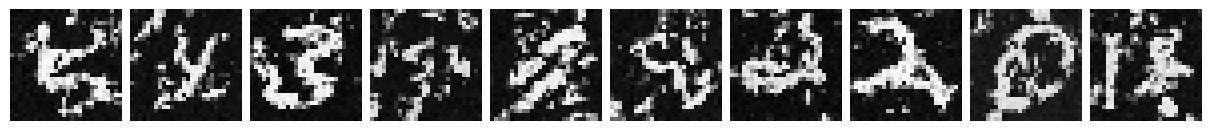}
        \caption{Scale = 2.0}
        \label{fig:sfig1}
    \end{subfigure}
    \caption{Image Samples with Scaled Rademacher Noise}
    \label{fig:noise_scaling}
\end{figure}

\section{DISCUSSION}

In this work, we presented a rigorous yet simplified analysis of the Euler–Maruyama discretization error for sampling in diffusion models. By utilizing a time-rescaling approach and Grönwall’s inequality, we established a strong convergence rate of $\mathcal{O}(1/\sqrt{T})$ under standard Lipschitz continuity assumptions. This framework offers a more accessible alternative to existing complex proofs, facilitating future theoretical explorations into the convergence properties of generative SDEs.

A key contribution of our study is the theoretical and empirical validation of noise substitution. We demonstrated that the standard Gaussian noise in the reverse process can be replaced by computationally cheaper discrete alternatives, such as Rademacher or Uniform noise, without compromising the convergence rate or sample quality. This interchangeability has profound implications for the deployment of diffusion models on resource-constrained hardware. By eliminating the need for expensive Gaussian random number generators (RNGs) and enabling the use of simple bitwise operations or integer arithmetic for noise injection, our findings pave the way for highly efficient FPGA and mobile-edge implementations.

Furthermore, our results shed light on the fundamental robustness of the diffusion sampling process. The empirical parity between Gaussian and symmetric discrete noises—observed across both MNIST and CIFAR-10 datasets—suggests that the generative capability of these models is driven primarily by the first two moments of the driving noise process, rather than the specific shape of the distribution. This aligns with the "universality" often seen in high-dimensional probability theory. However, the degradation observed with asymmetric noise (Laplace) highlights a critical boundary condition: while the distribution shape is flexible, symmetry remains a necessary condition for minimizing discretization error in the standard VP-SDE formulation.

Future work may extend this analysis to lower-precision quantization regimes, exploring the limits of noise simplification before the model's performance deteriorates. Furthermore, applying this simplified proof technique to other SDE-based paradigms, such as Consistency Models or Flow Matching, presents a promising direction for unifying the theoretical landscape of generative modeling.

\section*{Acknowledgements}
This work was supported by the New Faculty Startup Fund from Seoul National University.

\bibliography{references}

@InProceedings{pmlr-v202-song23a,
  title = 	 {Consistency Models},
  author =       {Song, Yang and Dhariwal, Prafulla and Chen, Mark and Sutskever, Ilya},
  booktitle = 	 {Proceedings of the 40th International Conference on Machine Learning},
  pages = 	 {32211--32252},
  year = 	 {2023},
  editor = 	 {Krause, Andreas and Brunskill, Emma and Cho, Kyunghyun and Engelhardt, Barbara and Sabato, Sivan and Scarlett, Jonathan},
  volume = 	 {202},
  series = 	 {Proceedings of Machine Learning Research},
  month = 	 {23--29 Jul},
  publisher =    {PMLR},
  url = 	 {https://proceedings.mlr.press/v202/song23a.html},

}

@inproceedings{KingmaW13,
  author       = {Diederik P. Kingma and
                  Max Welling},
  editor       = {Yoshua Bengio and
                  Yann LeCun},
  title        = {Auto-Encoding Variational Bayes},
  booktitle    = {2nd International Conference on Learning Representations, {ICLR} 2014,
                  Banff, AB, Canada, April 14-16, 2014, Conference Track Proceedings},
  year         = {2014},
}

@inproceedings{GoodfellowPMXWOCB14,
  author       = {Ian J. Goodfellow and
                  Jean Pouget{-}Abadie and
                  Mehdi Mirza and
                  Bing Xu and
                  David Warde{-}Farley and
                  Sherjil Ozair and
                  Aaron C. Courville and
                  Yoshua Bengio},
  title        = {Generative Adversarial Nets},
  booktitle    = {Advances in Neural Information Processing Systems 27: Annual Conference
                  on Neural Information Processing Systems 2014, December 8-13 2014,
                  Montreal, Quebec, Canada},
  pages        = {2672--2680},
  year         = {2014},
 
}

@inproceedings{chen2023improved,
  author    = {Hongrui Chen and Holden Lee and Jianfeng Lu},
  title     = {Improved Analysis of Score-based Generative Modeling: User-friendly Bounds under Minimal Smoothness Assumptions},
  booktitle = {Proceedings of the 40th International Conference on Machine Learning},
  volume    = {202},
  pages     = {4735--4763},
  year      = {2023},
  publisher = {PMLR},
  url       = {https://proceedings.mlr.press/v202/chen23q.html}
}

@article{Li2023,
  author    = {Gen Li and Yuting Wei and Yuxin Chen and Yuejie Chi},
  title     = {Towards Faster Non-Asymptotic Convergence for Diffusion-Based Generative Models},
  journal   = {arXiv preprint arXiv:2306.09251},
  year      = {2023},
  url       = {https://arxiv.org/abs/2306.09251}
}

@inproceedings{lee2023convergence,
  title     = {Convergence of Score-Based Generative Modeling for General Data Distributions},
  author    = {Holden Lee and Jianfeng Lu and Yixin Tan},
  booktitle = {Proceedings of the 34th International Conference on Algorithmic Learning Theory (ALT)},
  year      = {2023},
  url       = {https://arxiv.org/abs/2209.12381}
}

@inproceedings{chen2023sampling,
  title     = {Sampling is as Easy as Learning the Score: Theory for Diffusion Models with Minimal Data Assumptions},
  author    = {Chen, Sitan and Chewi, Sinho and Li, Jungshian and Li, Yuanzhi and Salim, Adil and Zhang, Anru R.},
  booktitle = {Proceedings of the 11th International Conference on Learning Representations (ICLR)},
  year      = {2023},

}

@article{yang2023diffusion,
  title={Diffusion models: A comprehensive survey of methods and applications},
  author={Yang, L. and Zhang, Z. and Song, Y. and Hong, S. and Xu, R. and Zhao, Y. and Zhang, W. and Cui, B. and Yang, M.-H.},
  journal={ACM Computing Surveys},
  volume={56},
  number={4},
  pages={1--39},
  year={2023},
  publisher={ACM}
}

@article{tang2024tutorial,
  title={Score-based diffusion models via stochastic differential equations--a technical tutorial},
  author={Tang, W. and Zhao, H.},
  journal={arXiv preprint arXiv:2402.07487},
  year={2024},
  url={https://arxiv.org/abs/2402.07487}
}

@article{jiao2025instance,
  title={Instance-dependent Convergence Theory for Diffusion Models},
  author={Jiao, Yuchen and Li, Gen},
  journal={arXiv preprint arXiv:2410.13738},
  year={2025},
  url={https://doi.org/10.48550/arXiv.2410.13738}
}

@inproceedings{benton2023,
  author    = {Joe Benton and Valentin De Bortoli and Arnaud Doucet and George Deligiannidis},
  title     = {Nearly $d$-Linear Convergence Bounds for Diffusion Models via Stochastic Localization},
  booktitle = {The Twelfth International Conference on Learning Representations (ICLR)},
  year      = {2023},
}

@article{li2024c,
  author    = {Gen Li and Yuting Wei and Yuejie Chi and Yuxin Chen},
  title     = {A Sharp Convergence Theory for the Probability Flow ODEs of Diffusion Models},
  journal   = {arXiv preprint arXiv:2408.02320},
  year      = {2024},
  note      = {2024c},
  url       = {https://arxiv.org/abs/2408.02320}
}

@article{ramesh2022hierarchical,
  title={Hierarchical Text-Conditional Image Generation with CLIP Latents},
  author={Ramesh, Aditya and Dhariwal, Prafulla and Nichol, Alex and Chu, Casey and Chen, Mark},
  journal={arXiv preprint arXiv:2204.06125},
  year={2022}
}

@inproceedings{kong2021diffwave,
  title={DiffWave: A Versatile Diffusion Model for Audio Synthesis},
  author={Kong, Zhifeng and Ping, Wei and Huang, Jiaji and Zhao, Kexin and Catanzaro, Bryan},
  booktitle={International Conference on Learning Representations (ICLR)},
  year={2021}
}

@article{trippe2022diffusion,
  title={Diffusion probabilistic modeling of protein backbones in 3D for the motif-scaffolding problem},
  author={Trippe, Brian L. and others},
  journal={arXiv preprint arXiv:2206.04119},
  year={2022}
}

@article{watson2023novo,
  title={De novo design of protein structure and function with RFdiffusion},
  author={Watson, James L. and others},
  journal={Nature},
  volume={620},
  number={7975},
  pages={824--831},
  year={2023},
  publisher={Nature Publishing Group}
}

@article{ho2022imagenvideo,
  title={Imagen Video: High Definition Video Generation with Diffusion Models},
  author={Ho, Jonathan and others},
  journal={arXiv preprint arXiv:2210.02303},
  year={2022}
}

@article{wolleb2022diffusion,
  title={Diffusion models for medical anomaly detection},
  author={Wolleb, Julia and Van Gool, Luc and Schmid, Jonas and Davatzikos, Christos and Candemir, Sinan and Sandk\"uhler, Martin},
  journal={arXiv preprint arXiv:2203.04391},
  year={2022}
}

@inproceedings{sohl2015deep,
  title={Deep unsupervised learning using nonequilibrium thermodynamics},
  author={Sohl-Dickstein, Jascha and Weiss, Eric and Maheswaranathan, Niru and Ganguli, Surya},
  booktitle={International Conference on Machine Learning},
  pages={2256--2265},
  year={2015},
  organization={PMLR}
}

@inproceedings{ho2020denoising,
  title={Denoising diffusion probabilistic models},
  author={Ho, Jonathan and Jain, Ajay and Abbeel, Pieter},
  booktitle={Advances in Neural Information Processing Systems},
  volume={33},
  pages={6840--6851},
  year={2020}
}

@inproceedings{song2021denoising,
  title={Denoising diffusion implicit models},
  author={Song, Jiaming and Meng, Chenlin and Ermon, Stefano},
  booktitle={International Conference on Learning Representations},
  year={2021}
}

@inproceedings{saharia2022photorealistic,
  title={Photorealistic text-to-image diffusion models with deep language understanding},
  author={Saharia, Chitwan and Chan, William and Saxena, Saurabh and Li, Lala and Whang, Jay and Denton, Emily and Salimans, Tim and Ho, Jonathan and Fleet, David J and Norouzi, Mohammad},
  booktitle={Advances in Neural Information Processing Systems},
  year={2022}
}

@article{ho2022classifier,
  title={Classifier-free diffusion guidance},
  author={Ho, Jonathan and Salimans, Tim and Chan, William and Ramachandran, Prajit and Chen, Xi},
  journal={arXiv preprint arXiv:2207.12598},
  year={2022}
}

@article{song2021score,
  title={Score-based generative modeling through stochastic differential equations},
  author={Song, Yang and Sohl-Dickstein, Jascha and Kingma, Diederik and Kumar, Abhishek and Ermon, Stefano and Poole, Ben},
  journal={ICLR},
  year={2021}
}

@inproceedings{xiao2022discrete,
  title={On Discrete Denoising Score Matching for Practical Learning of Diffusion Models},
  author={Xiao, Lin and Meng, Chenlin and Song, Yang and Ermon, Stefano},
  booktitle={Advances in Neural Information Processing Systems (NeurIPS)},
  year={2022}
}

@article{Gronwall,
 ISSN = {0003486X, 19398980},
 URL = {http://www.jstor.org/stable/1967124},
 author = {T. H. Gronwall},
 journal = {Annals of Mathematics},
 number = {4},
 pages = {292--296},
 publisher = {[Annals of Mathematics, Trustees of Princeton University on Behalf of the Annals of Mathematics, Mathematics Department, Princeton University]},
 title = {Note on the Derivatives with Respect to a Parameter of the Solutions of a System of Differential Equations},
 urldate = {2025-05-29},
 volume = {20},
 year = {1919}
}

@inproceedings{yoon2023score,
  title={Score-based Generative Models with {L}{\'e}vy Processes},
  author={Yoon, Hyeongju and Park, Eunwoo and Ryu, Ernest K},
  booktitle={Advances in Neural Information Processing Systems},
  volume={36},
  year={2023}
}

@inproceedings{pandey2025heavy,
  title={Heavy-Tailed Diffusion Models},
  author={Pandey, V and others},
  booktitle={International Conference on Learning Representations},
  year={2025}
}

@inproceedings{nachmani2021denoising,
  title={Denoising Diffusion Gamma Models},
  author={Nachmani, Eliya andRomanowitz, Robin and Wolf, Lior},
  booktitle={International Conference on Learning Representations},
  year={2021}
}

@inproceedings{bansal2022cold,
  title={Cold Diffusion: Inverting Arbitrary Image Transforms Without Noise},
  author={Bansal, Arpit and Borgnia, Eitan and Chu, Hong-Min and Li, Jie S and Kazemi, Hamid and Huang, Furong and Goldblum, Micah and Geiping, Jonas and Goldstein, Tom},
  booktitle={Proceedings of the IEEE/CVF Conference on Computer Vision and Pattern Recognition},
  year={2022}
}

@inproceedings{daras2022soft,
  title={Soft Diffusion: Score Matching for General Corruptions},
  author={Daras, Giannis and Delbracio, Mauricio and Talebi, Hossein and Dimakis, Alexandros G and Milanfar, Peyman},
  booktitle={Advances in Neural Information Processing Systems},
  volume={35},
  year={2022}
}
\bibliographystyle{plainnat}

\newpage
\appendix
\onecolumn
\section{EXPERIMENT SETTINGS}

\subsection{Performance of Various Noises}

The alternative noises used in our experiments are defined as below. $\mathsf{torch}$ and $\mathsf{math}$ modules are used.

\begin{lstlisting}[language=Python, caption={Discrete Gaussian Noise}]
def DG_matrix(shape, device = None):
    if device is None:
        device = torch.device('cuda' if torch.cuda.is_available() else 'cpu')
    probs = torch.rand(shape, device = device)
    matrix = torch.zeros_like(probs)

    sqrt3 = math.sqrt(3)
    matrix[probs < 1/6] = -sqrt3
    matrix[(probs >= 1/6) & (probs < 1/3)] = sqrt3
    return matrix
\end{lstlisting}

\begin{lstlisting}[language=Python, caption={Uniform Noise}]
def UD_matrix(shape, device=None):
    if device is None:
        device = torch.device('cuda' if torch.cuda.is_available() else 'cpu')
    
    a = -math.sqrt(3)
    b = math.sqrt(3)
    
    matrix = torch.empty(shape, device=device).uniform_(a, b)
    return matrix
\end{lstlisting}

\begin{lstlisting}[language=Python, caption={Rademacher Noise}]
def Rademacher_matrix(shape, device=None):
    if device is None:
        device = torch.device('cuda' if torch.cuda.is_available() else 'cpu')
    matrix = 2 * torch.randint(0, 2, shape, device=device).float() - 1
    return matrix
\end{lstlisting}

\begin{lstlisting}[language=Python, caption={Laplace Noise}]
def Laplace_matrix(shape, device=None):
    if device is None:
        device = torch.device('cuda' if torch.cuda.is_available() else 'cpu')
    b = 1 / math.sqrt(2)
    u = torch.rand(shape, device=device) - 0.5
    matrix = -b * torch.sign(u) * torch.log1p(-2 * u.abs())
    return matrix
\end{lstlisting}

\begin{lstlisting}[language=Python, caption={Triangular Noise}]
def Triangular_matrix(shape, device=None):
    if device is None:
        device = torch.device('cuda' if torch.cuda.is_available() else 'cpu')
    a = -math.sqrt(6)
    b = 0
    c = math.sqrt(6)
    u = torch.rand(shape, device=device)
    matrix = torch.where(
        u < (b - a) / (c - a),
        a + torch.sqrt(u * (b - a) * (c - a)),
        c - torch.sqrt((1 - u) * (c - b) * (c - a))
    )
    return matrix
\end{lstlisting}

\begin{lstlisting}[language=Python, caption={Arcsine Noise}]
def Arcsine_matrix(shape, device=None):
    if device is None:
        device = torch.device('cuda' if torch.cuda.is_available() else 'cpu')
    u = torch.rand(shape, device=device)
    matrix = math.sqrt(2) * torch.sin(math.pi * (u - 0.5))
    return matrix
\end{lstlisting}

\subsection{Evaluation}

The evaluation of each experiments are done using Fr\'echt Inception Distance (FID) score and CLIP-Maximum Mean Discrepancy (CMMD) score. For FID score evaluation, we used the compute\_fid function in cleanfid module. We used CLIPModel and CLIPProcessor in transformer module to extract CLIP features from images and compute CMMD score.

\end{document}